\def\year{2021}\relax
\documentclass[letterpaper]{article} 
\usepackage{aaai21}  
\usepackage{times}  
\usepackage{helvet} 
\usepackage{courier}  
\usepackage[hyphens]{url}  
\usepackage{graphicx} 
\usepackage{graphicx}  
\usepackage{natbib}  
\usepackage{caption} 
\usepackage{booktabs} 

\usepackage{amsmath,amssymb}
\newtheorem{definition}{Definition}
\newtheorem{theorem}{Theorem}
\newtheorem{lemma}{Lemma}
\newtheorem{claim}{Claim}

\newenvironment{proof}{\noindent{\sf Proof.}}{\hfill $\Box\hspace{2mm}$\linebreak}
\renewcommand{\phi}{\varphi}
\renewcommand{\epsilon}{\varepsilon}

\newenvironment{proof-of-claim}{\noindent{\sc Proof of Claim.}}{\hfill $\Box\hspace{2mm}$\linebreak}

\mathchardef\mhyphen="2D 

\renewcommand{\phi}{\varphi}
\renewcommand{\epsilon}{\varepsilon}
\newcommand{\qed}{\hfill $\Box\hspace{2mm}$}

\newcommand{\C}{{\sf C}}
\newcommand{\K}{{\sf K}}
\usepackage{dialogue}

\newsavebox{\diamonddotsavebox}
\sbox{\diamonddotsavebox}{$\Diamond$\hspace{-1.8mm}\raisebox{0.3mm}{$\cdot$}\hspace{1mm}}

\author{}

\title{Comprehension and Knowledge}

\author {
    Pavel Naumov,\textsuperscript{\rm 1}
    Kevin Ros\textsuperscript{\rm 2}\\
}
\affiliations {
    \textsuperscript{\rm 1} King's College \\
    \textsuperscript{\rm 2} University of Illinois at Urbana-Champaign \\
    pgn2@cornell.edu, kjros2@illinois.edu
}

\begin{document}

\maketitle

\begin{abstract}
The ability of an agent to comprehend a sentence is tightly connected to the agent's prior experiences and background knowledge. The paper suggests to interpret comprehension as a modality and proposes a complete bimodal logical system that describes an interplay between comprehension and knowledge modalities. 
\end{abstract}

\section{Introduction}

Natural language understanding is a well-developed area of Artificial Intelligence concerned with machine comprehension of human writing and speech. It has applications in machine translation, intelligent virtual assistant design, news-gathering, voice-activation, and sentiment analysis. Most of the current approaches to natural language understanding are based on machine learning techniques. In this paper we propose a logic-based framework for defining and reasoning about comprehension.

Comprehension often requires elimination of the ambiguity present in natural language. This often can be done by taking into account the background knowledge. As an example, consider the following dialog that took place on January 25, 1990 near John F. Kennedy International Airport in New York:
\begin{dialogue}
\speak{Air Traffic Controller} Avianca 052 heavy I'm gonna bring you about fifteen miles north east and then turn you back onto the approach is that fine with you and your fuel
\speak{First Officer} I guess so thank you very much
\end{dialogue}
About 8 minutes after this conversation, Avianca flight 052 ran out of fuel and crashed. Out of 158 persons aboard, 73 died~\cite[p.v]{90ntsb}.
In its report, National Transportation Safety Board lists ``the lack of standardized understandable terminology'' as a contributing factor to the crash~\cite[p.v]{90ntsb}. While analysing the crash, Helmreich points out that Colombia and the United States score very differently on such cultural dimensions as power distance, individualism-collectivism, and uncertainty avoidance. He argues that these cultural factors contributed to the lack of understanding between the Colombian crew and the American air traffic controller~\cite{h94ijap}; others agree~\cite{ofud97ccwg}.


In a low power distance culture ``I guess so'' is an informal way to confirm that the aircraft has enough fuel while, perhaps, communicating the crew's unhappiness to make another loop in the air. In a high power distance culture, such as Colombia, it would be too disrespectful to express the same idea with ``I guess so''. Instead, in such cultures, ``I guess so'' is a mitigated expression of a concern, a respectful way to warn about an imminent danger. The United States, where this sentence could be interpreted either way\footnote{When American air traffic controllers were asked by the investigators what words they would respond immediately when a flight crew communicates a low fuel emergency, they replied ``MAYDAY'', ``PAN, PAN, PAN'', and ``Emergency''~\cite[p.63]{90ntsb}.  Avianca 052 communication transcripts show that the word ``Emergency'' was used in the communication between the pilot and the first officer, but not with the air traffic controller~\cite[p.10]{90ntsb}.}, falls in the middle of power distance scale~\cite[p.87]{hofstede01}.



Note that this ambiguity disappears if the controller has additional knowledge about the cultural background of the crew. As the example shows, knowledge might play a key role in comprehension. In this paper we propose a logic that describes the interplay between knowledge and comprehension.




The rest of this paper in structured as follows. First, we define  a model of our logical system and relate this model to the above example. Then, we define the syntax and the formal semantics of our system, give one more example, and review the related literature. Next, we show that the two modalities of our logical system, knowledge and comprehension,  can not be expressed through each other and list the axioms of our logical system. In the two sections that follow, we prove soundness and sketch the proof of completeness of our system. The full proof of completeness as well as a discussion of how our definition of comprehension can be adapted to settings where meanings and states have probabilities are in the
full version of this paper~\cite{nr20arxiv}.

\section{Epistemic Model with Meanings}

We define knowledge and comprehension in the context of a given epistemic model with meanings.

\begin{definition}\label{model}
An epistemic model with meanings is any tuple 
$$(W,\{\sim_a\}_{a\in\mathcal{A}},\{M_w\}_{w\in W},\{\pi_w\}_{w\in W})$$
such that
\begin{enumerate}
    \item $W$ is an arbitrary set of ``states'',
    \item $\sim_a$ is an ``indistinguishability'' equivalence relation on set $W$ for each agent $a\in\mathcal{A}$,
    \item $M_w$ is a set of ``meanings'' for each state $w\in W$,
    \item $\pi_w$ is a function from propositional variables into subsets of $M_w$ for each state $w\in W$. 
\end{enumerate}
\end{definition}

As we discussed in the introduction, locution ``I guess so'' is a real-world example of the kind of ambiguity that an artificial agent should be able to reason about  in order to comprehend human verbal communication. In this section we interpret it as a statement that words ``I guess so'' give an accurate description of the current state. We denote this statement by propositional variable $p$. 

\begin{figure}[ht]
\begin{center}
\scalebox{0.4}{\includegraphics{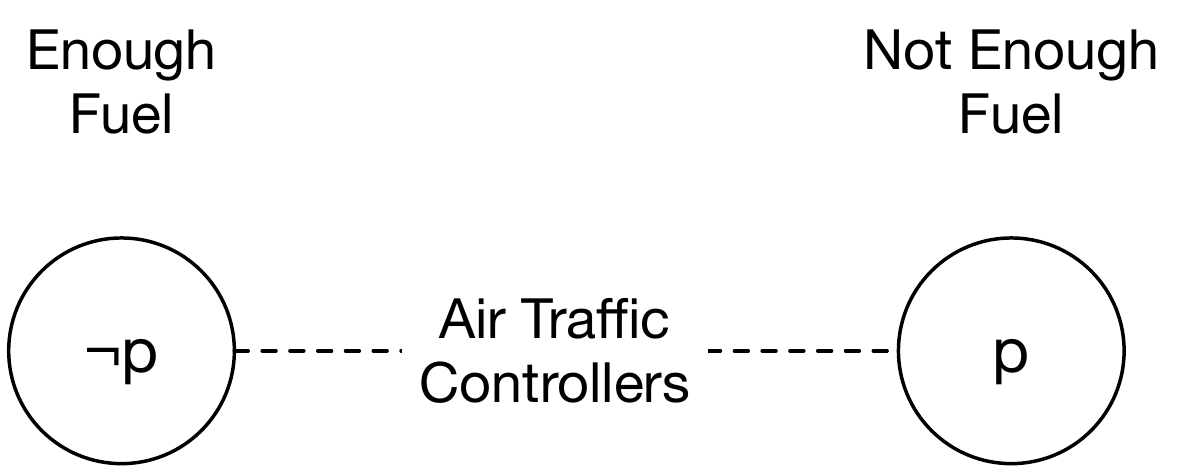}}
\caption{Landing in Bogot\'a, Columbia.}\label{intro bogota figure}
\end{center}
\end{figure}
Figure~\ref{intro bogota figure} depicts an epistemic model capturing a hypothetical landing of Avianca 052 in Bogot\'a, Columbia, where the flight originated. This model has two states, ``Enough Fuel'' and ``Not Enough Fuel'', indistinguishable (before the pilots say ``I guess so'') to the air traffic controllers. Since the traffic controllers at Bogot\'a airport have the same high power distance cultural background as Avianca's pilots, to them statement $p$ is true in state ``Not Enough Fuel'' and false in state ``Enough Fuel''. Once the Bogot\'a controllers hear ``I guess so'', they likely will conclude that the plane is low on fuel and issue an emergency landing order.

\begin{figure}[ht]
\begin{center}
\scalebox{0.4}{\includegraphics{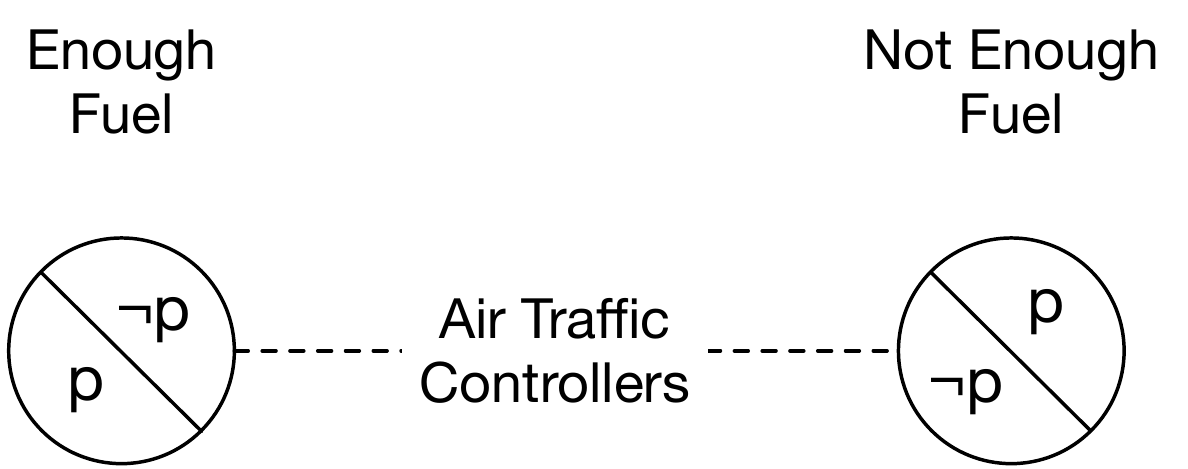}}
\caption{Landing in New York, USA.}\label{intro nyc figure}
\end{center}
\end{figure}

Figure~\ref{intro nyc figure} depicts an epistemic model describing the actual landing of Avianca 052 at JFK International Airport in New York. It also has two states indistinguishable to the air traffic controllers. We capture the ambiguity of the locution ``I guess so'' to New York controllers by saying that it has two distinct meanings: a low-power-distance culture meaning ``I am ok on fuel, but I am unhappy about another loop in the air'' and high-power-distance culture meaning of a mitigated expression of a concern. Both of these meanings exist in either of the two states. In general, we use meanings to capture ambiguity of a natural language. Just as words and phrases could be interpreted differently in various contexts, the same propositional variable could be true in a given state under one meaning and false under another. We allow for different sets of meanings in different states. We visually represent the meanings by dividing the circle of the state into two (or more) areas corresponding to the different meanings. In the diagram, the upper-right and lower-left semicircles represents the high and the low power distance meanings respectively.

\section{Syntax and Semantics}

In this section we describe the syntax and the formal semantics of our logical system. We assume a fixed countable set of propositional variables and a fixed countable set of agents $\mathcal{A}$. The language $\Phi$ of our system is defined by the grammar
$$
\phi := p\;|\;\neg\phi\;|\;\phi\to\phi\;|\;\K_a\phi\;|\;\C_a\phi,
$$
where $p$ is a propositional variable and $a\in\mathcal{A}$ is an agent. We read $\K_a\phi$ as ``agent $a$ knows $\phi$'' and $\C_a\phi$ as ``agent $a$ comprehends $\phi$''. We assume that conjunction $\wedge$, biconditional $\leftrightarrow$, and true $\top$ are defined through negation $\neg$ and implication $\to$ in the usual way. For any finite set of formulae $Y\subseteq\Phi$, by $\wedge Y$ we mean the conjunction of all formulae in set $Y$. By definition, $\wedge\varnothing$ is formula $\top$.

\begin{definition}\label{sat}
For any formula $\phi\in\Phi$, any state $w\in W$, and any meaning $m\in M_w$, satisfaction relation $(w,m)\Vdash\phi$ is defined recursively as follows:
\begin{enumerate}
    \item $(w,m)\Vdash p$ if $m\in\pi_w(p)$,
    \item $(w,m)\Vdash\neg\phi$ if $(w,m)\nVdash\phi$,
    \item $(w,m)\Vdash\phi\to\psi$ if $(w,m)\nVdash\phi$ or $(w,m)\Vdash\psi$,
    \item $(w,m)\Vdash\K_a\phi$ if $(u,m')\Vdash\phi$ for each state $u\in W$ such that $w\sim_a u$ and each meaning $m'\in M_u$,
    \item $(w,m)\Vdash\C_a\phi$ when for each state $u\in W$ and any  meanings $m',m''\in M_u$, if $w\sim_a u$ and $(u,m')\Vdash\phi$, then $(u,m'')\Vdash\phi$. 
\end{enumerate}
\end{definition}
Note that one can potentially consider the following alternative to item 4 of the above definition:

{\em
\begin{enumerate}
    \item[4$'$.] $(w,m)\Vdash\K_a\phi$ if $(u,m)\Vdash\phi$ for each state $u\in W$ such that $w\sim_a u$.
\end{enumerate}
}
Under this definition, statement $(w,m)\Vdash\K_a\phi$ would mean that ``agent $a$ knows that $\phi$ is true in state $w$ under meaning $m$''.  An agent might know $\phi$ to be true under one meaning and false under another. If statement $\phi$ is written or said by somebody else, the agent will not know if it is true or false. Thus, we stipulate that in order for an agent to know that $\phi$ is true, she should know that $\phi$ is true under any meaning. This is captured in item 4 of Definition~\ref{sat}.  

Item 5 of Definition~\ref{sat} is the key definition of this paper. It formally specifies the semantics of the comprehension modality $\C$. As defined in item 4, statement ``an agent $a$ knows $\phi$'' means that $\phi$ is {\em true} under each meaning in each $a$-indistinguishable state. We say that agent $a$ comprehends $\phi$ if $\phi$ is {\em consistent across the meanings} in each $a$-indistinguishable state. In other words, $a$ comprehends $\phi$ if, for each $a$-indistinguishable state, $\phi$ is true under one meaning if and only if it is true under any other meaning. 

In our example from Figure~\ref{intro bogota figure}, statement 
\begin{equation}\label{K example equation}
    \K_{\mbox{\scriptsize Traffic Controllers}}\;p
\end{equation}
is false in both states because $p$ is true under the unique meaning in the right state and is false under the unique meaning in the left state. At the same time, statement
\begin{equation}\label{C example equation}
    \C_{\mbox{\scriptsize Traffic Controllers}}\;p
\end{equation}
is true in both states because in both states the value of the propositional variable $p$ is vacuously consistent across all meanings in the state. In other words, in the example from Figure~\ref{intro bogota figure}, the air traffic controllers do not know (before the pilots say ``I guess so'') if statement $p$ is true or not, but they comprehend this statement due to the lack of multiple meanings.

In the example depicted in Figure~\ref{intro nyc figure}, statement~(\ref{K example equation}) is still false in both indistinguishable states because the diagram has at least one meaning in one of the states where propositional variable $p$ is false. In addition, statement~(\ref{C example equation}) is also false in both states of this example because the value of the propositional variable $p$ is not consistent across the meanings in at least one (in our case, both) of the two indistinguishable states.

To summarize, before the pilots say ``I guess so'', in both examples the air traffic controllers do not know if statement $p$ is true or not. However, the controllers in Bogot\'a comprehend $p$ and the controllers in New York do not.

The next lemma holds because, by item 5 of Definition~\ref{sat}, validity of  $(w,m)\Vdash\C_a\phi$ does not depend on the value $m$.
\begin{lemma}\label{C transfer lemma}
$(w,m)\Vdash\C_a\phi$ iff $(w,m')\Vdash\C_a\phi$ for any state $w\in W$ and any meanings $m,m'\in M_w$. \qed
\end{lemma}

\section{Guard Ava}

As another example, consider a hypothetical company whose office building has three entrances: $X$, $Y$, and $Z$.  Before the building opens, a robotic guard Ava is given the instruction ``{\em All visitors must enter the building through door X or through door Y and wear a badge}''. The door $X$ in the building is often broken and closed for repair, but Ava has access to the door status information. She knows that today the door is open.

To keep the formal model simple, let us assume that only one visitor will arrive today. Thus, a state of the model could be completely described by specifying (i) whether door $X$ is closed or open, (ii) whether this door is used by the visitor, and (iii) whether the visitor has a badge. Since the visitor cannot enter the building through a closed door, there are 10 different states in our model. These states are depicted in Figure~\ref{door figure}. In 4 of these states denoted by {\em double} circles, door $X$ is closed. In the remaining 6 states the door is open. The door used by the visitor and the badge status are represented by the row and column in which the state is located. For example, in state $w$, visitor enters through door $Y$ and wears a badge.

\begin{figure}[ht]
\begin{center}
\scalebox{0.4}{\includegraphics{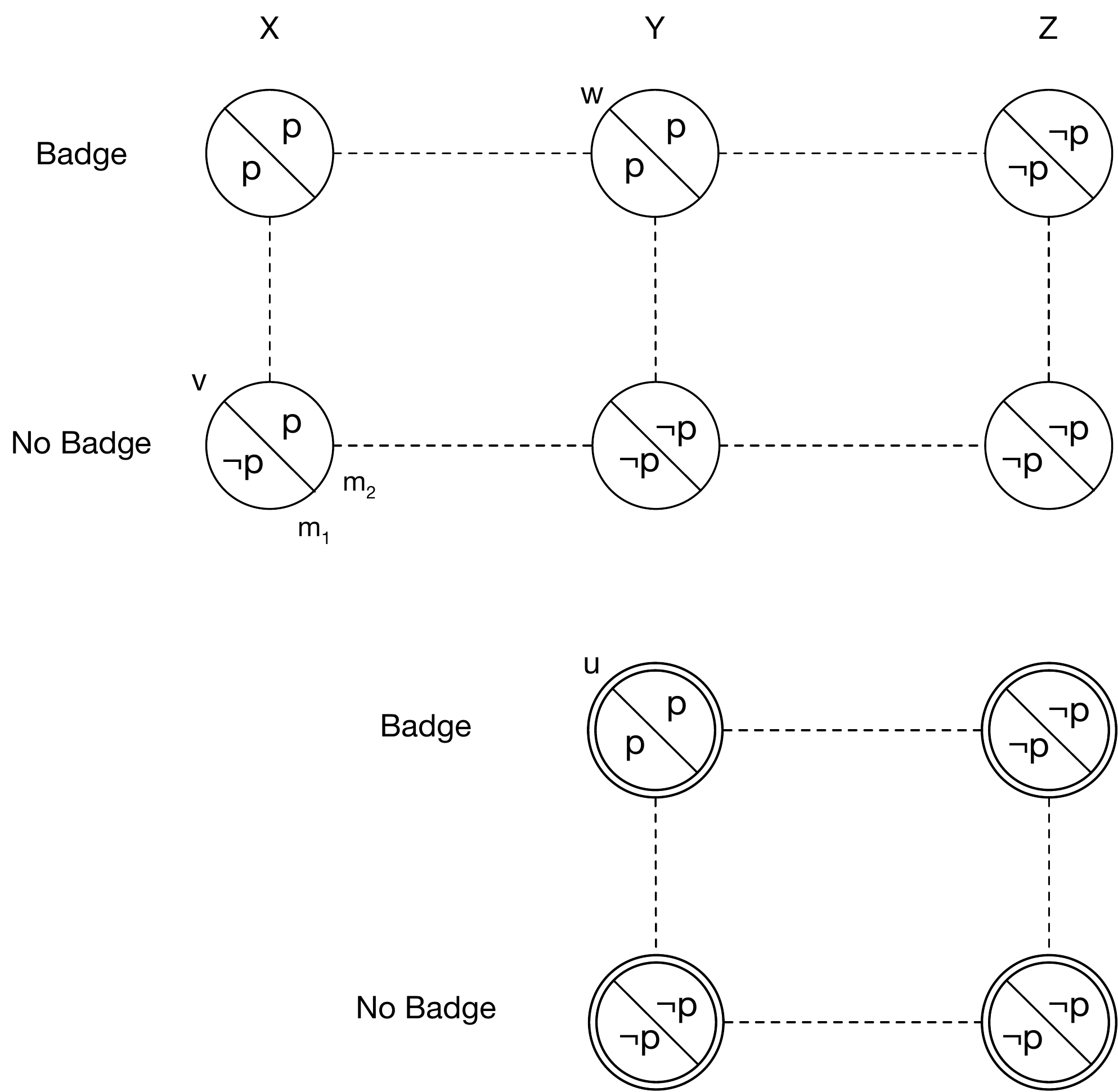}}
\caption{Door $X$ is closed in double-circled states.}\label{door figure}
\end{center}
\end{figure}

In this example we consider Ava's knowledge and comprehension of the above instruction {\em before} the visitor enters the building. Thus, she can distinguish the states with closed door $X$ from the states where door $X$ is open, but she does not know through which door the visitor will enter or whether he will wear a badge. This indistinguishability relation is depicted in the diagram by dashed lines between the states. 

Suppose that $x$ is proposition ``the visitor enters through door $X$'', $y$ is proposition ``the visitor enters through door $Y$'', and $b$ is proposition ``the visitor wears a badge''. Then, statement ``the visitor enters the building through door X or through door Y and wears a badge'' could be interpreted as either $(x\vee y)\wedge b$ or $x\vee (y\wedge b)$. These two interpretations are captured in our model as two different meanings. The meaning $m_1$ (low-left half-circle of each state) corresponds to the interpretation $(x\vee y)\wedge b$ and meaning $m_2$ (upper-right half-circle of each state) corresponds to the interpretation $x\vee (y\wedge b)$.

In state $v$, see Figure~\ref{door figure}, the  visitor, wearing no badge, enters the building through open door $X$.
Thus, in state $v$, statement  $(x\vee y)\wedge b$ is false, but statement $x\vee (y\wedge b)$ is true. In other words, in this state propositional variable $p$, representing statement ``the visitor enters the building through door X or through door Y and wears a badge'', is false under meaning $m_1$ and true under meaning $m_2$. Using our formal notations, $(v,m_1)\nVdash p$ and $(v,m_2)\Vdash p$. Hence, in state $v$ propositional variable $p$ is not consistent across the meanings. It is easy to see that propositional variable $p$ is consistent across the meanings in all other states of our model, see Figure~\ref{door figure}.

By Definition~\ref{sat}, in order for Ava to comprehend statement $p$ in, say, state $w$, this statement must be consistent across the meanings in all states indistinguishable from state $w$. Since it is not consistent across the meanings in state $v$, Ava does not comprehend statement $p$ in state $w$. Formally, $(w,m)\nVdash\C_{\mbox{\scriptsize Ava}}\; p$ for each meaning $m$ in state $w$. 

Finally, consider state $u$ in the same diagram. Here, just like in state $w$, the visitor enters through door $Y$ and wears a badge, however, this time door $X$ is closed and Ava knows about this. Note that statement $p$ is consistent across the meanings in all states indistinguishable from state $u$. Thus, $(u,m)\Vdash\C_{\mbox{\scriptsize Ava}}\; p$ for each meaning $m$ in state $u$. In other words, on the days when door $X$ is open Ava cannot comprehend sentence $p$, but she can comprehend it on the days when the door is closed. Having the additional knowledge that the visitor cannot enter through door $X$ improves her comprehension.

\section{Literature Review}

Langer states that ``the knowledge and experience an individual brings to a reading task are critical factors in comprehension'' (\citeyear{l84rrq}). The connection between knowledge and comprehension has long been a subject of psychology  and literacy studies~\cite{phg79jrb,kbbb00ps,hhbp04science,kae17orec}.

Within the field of psychology, the comprehension of logical connectives is investigated in~\cite{p73jecp}. D'Hanis suggests to use adaptive logic for capturing metaphors~(\citeyear{d02lcamr}). Another logical system for metaphors in Chinese language is advocated in~\cite{zz04jcip}.
Neither of the last two papers claim a complete axiomatization.

Comprehension can be viewed as a very special form of ``awareness closed under subformulae'' modality from Logic of General Awareness~\cite{fh87ai}. This connection is not very deep, although, as most of our axioms are not valid in that logic. We are not aware of any works proposing logical systems specifically for comprehension as a modality. Comprehension of a sentence could be thought of as the knowledge of what the sentence means. Thus, it is related to the other forms of knowledge, such as know-whether, know-what, know-how, know-why, know-who, know-where, and know-value \cite{w18hintikka}. 

Logics of {\em know-how} without knowledge modality are proposed in~\cite{w17synthese} and ~\cite{lw17icla}. A logic of know-how for a single agent that also contains a knowledge modality is introduced in~\cite{fhlw17ijcai}. Coalition logic of know-how with individual knowledge modality is axiomatized in~\cite{aa16jlc}. Several versions of coalition know-how logics with distributed knowledge modality are described in~\cite{nt17aamas,nt18ai,nt18aamas,nt18aaai,cn20ai}.

Logics of {\em know-whether} are studied in~\cite{fwv15rsl,fgksv20arxiv}.
Different forms of {\em know-value} logics are investigated in~\cite{wf13ijcai,gw16aiml,vgw17icla}. 
Logic of {\em know-why} is proposed in~\cite{xws19synthese}.

\section{Undefinability of Comprehension through Knowledge}

In this section we prove that the comprehension modality $\C$ is not definable through knowledge modality $\K$. More precisely, we show that modality $C$ can not be expressed in the language $\Phi^{\mhyphen\C}$ defined by the grammar
$$
\phi := p\;|\;\neg\phi\;|\;\phi\to\phi\;|\;\K_a\phi.
$$
We prove this by constructing two models indistinguishable in language $\Phi^{\mhyphen\C}$, but distinguishable in the full language $\Phi$ of our logical system.  Without loss of generality, we can assume that the set of agents $\mathcal{A}$ consists of a single agent $a$ and the set of propositional variables contains a single propositional variable $p$. The two models that we use to prove undefinability are depicted in Figure~\ref{C-undefinability figure}.
\begin{figure}[ht]
\begin{center}
\scalebox{0.4}{\includegraphics{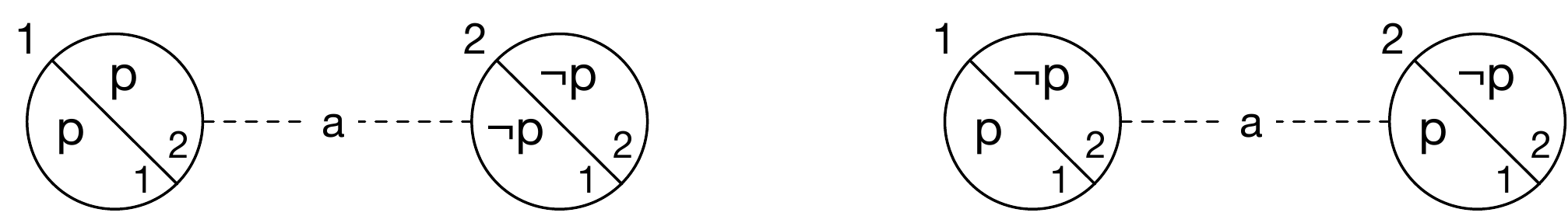}}
\caption{Two Models.}\label{C-undefinability figure}
\end{center}
\end{figure}
We refer to them as the left and the right models. Both models have two states: 1 and 2 indistinguishable to agent $a$. Each state has two meanings: 1 and 2. In the diagram, the number {\em outside} of a circle is the name of the state, while the number {\em inside} of a semi-circle is the name of the meaning. It will be important for our proof that states and meanings have the same names. Valuation functions $\pi^l$ of the left model and $\pi^r$ of the right model are specified in Figure~\ref{C-undefinability figure}. For example, $\pi^l_1(p)=\{1,2\}$. In other words, in state $1$ of the left model, propositional variable $p$ is true under meaning 1 and meaning 2.  By $\Vdash_l$ and $\Vdash_r$ we denote the satisfaction relation of the left and the right model respectively. The next lemma proves that the two models are indistinguishable in language $\Phi^{\mhyphen \C}$. Note that the order of $x$ and $y$ is different on the left-hand-side of the two satisfaction statements in this lemma. 
\begin{lemma}
$(x,y)\Vdash_l\phi$ iff $(y,x)\Vdash_r\phi$ for any integers $x,y\in\{1,2\}$ and any formula  $\phi\in\Phi^{\mhyphen \C}$.
\end{lemma}
\begin{proof}
We prove the statement by induction on structural complexity of formula $\phi$. First, we consider the case when $\phi$ is a propositional variable $p$. Observe that $y\in \pi^l_x(p)$ iff $x\in \pi^l_y(p)$ for any integers $x,y\in\{1,2\}$, see Figure~\ref{C-undefinability figure}. Thus, $(x,y)\Vdash_l p$ iff $(y,x)\Vdash_r p$ by item 1 of Definition~\ref{sat}.

If formula $\phi$ is a negation or an implication, then the required follows from items 2 and 3 of Definition~\ref{sat} and the induction hypothesis in the standard way.

Suppose that formula $\phi$ has the form $\K_a\psi$. By item 4 of Definition~\ref{sat}, statement $(x,y)\Vdash_l\K_a\psi$ implies that $(x',y')\Vdash_l\psi$ for any integers $x',y'\in\{1,2\}$. Hence, by the induction hypothesis, $(y',x')\Vdash_r\psi$ for any integers $x',y'\in\{1,2\}$. Therefore, $(y,x)\Vdash_r\K_a\psi$ again by item 4 of Definition~\ref{sat}. The proof in the other direction is similar. 
\end{proof}

The next lemma shows that the left and the right models are distinguishable in the language $\Phi$ of our logical system.

\begin{lemma}
$(1,1)\Vdash_l\C_a p$ and $(1,1)\nVdash_r\C_a p$.
\end{lemma}
\begin{proof}
Note that $1\in \pi_x^l(p)$ iff $2\in \pi_x^l(p)$ for any integer $x\in\{1,2\}$, see Figure~\ref{C-undefinability figure}. Thus, $(x,1)\Vdash_l p$ iff $(x,2)\Vdash_l p$ for any integer $x\in\{1,2\}$ by item 1 of Definition~\ref{sat}. Therefore, $(1,1)\Vdash_l\C_a p$ by item 5 of Definition~\ref{sat}.

Next, observe that $1\in \pi_1^r(p)$ and  $2\notin \pi_1^r(p)$, see Figure~\ref{C-undefinability figure}. Thus, $(1,1)\Vdash_r p$ and $(1,2)\nVdash_r p$ by item 1 of Definition~\ref{sat}. Therefore, $(1,1)\nVdash_r\C_a p$ by item 5 of Definition~\ref{sat}.
\end{proof}

The next theorem follows from the two lemmas above.

\begin{theorem}
Comprehension modality $\C$ is not definable in language $\Phi^{\mhyphen \C}$. \qed
\end{theorem}

\section{Undefinability of Knowledge through Comprehension}

In this section we prove that knowledge modality $\K$ is not definable in the language $\Phi^{\mhyphen\K}$ specified by the grammar
$$
\phi := p\;|\;\neg\phi\;|\;\phi\to\phi\;|\;\C_a\phi.
$$
The proof is similar to the one in the previous section. The left and the right models are depicted in Figure~\ref{K-undefinability figure}. \begin{figure}[ht]
\begin{center}
\scalebox{0.4}{\includegraphics{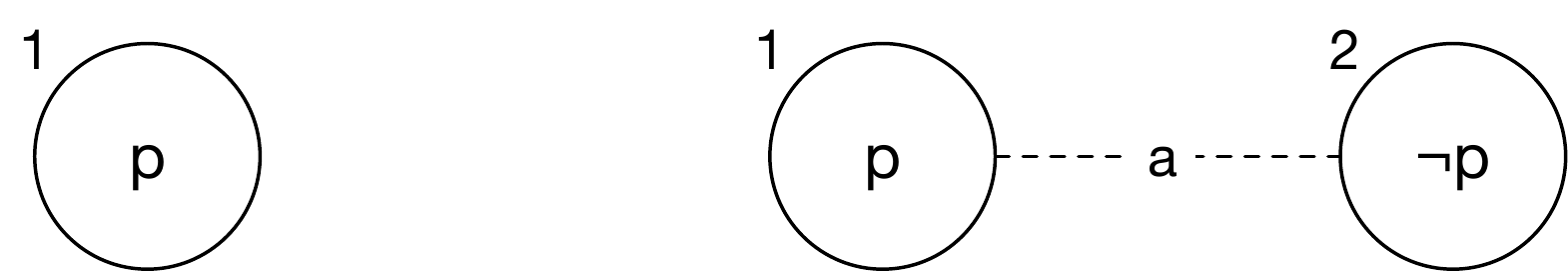}}
\caption{Two Models.}\label{K-undefinability figure}
\end{center}
\end{figure}
The left model has a single state 1, while the right model has two states, 1 and 2, indistinguishable to agent $a$. All states in both models have only one meaning, which we refer to as meaning 1. First, we show that state 1 in the left model is indistinguishable in language $\Phi^{\mhyphen\K}$ from state 1 in the right model. 

\begin{lemma}
$(1,1)\Vdash_l\phi$ iff $(1,1)\Vdash_r\phi$ for any $\phi\in\Phi^{\mhyphen \K}$.
\end{lemma}
\begin{proof}
We prove the statement of the lemma by induction on structural complexity of formula $\phi$. Note that $1\in \pi^l_1(p)$ and  $1\in \pi^r_1(p)$, see Figure~\ref{K-undefinability figure}. Thus, $(1,1)\Vdash_l p$ and $(1,1)\Vdash_r p$ by item 1 of Definition~\ref{sat}. Therefore, the statement of the lemma holds if formula $\phi$ is propositional variable $p$.

If formula $\phi$ is a negation or an implication, then the required follows from items 2 and 3 of Definition~\ref{sat} and the induction hypothesis in the standard way.

Suppose that formula $\phi$ has the form $\C_a\psi$. Note that $(1,1)\Vdash_l\C_a\psi$ by item 5 of Definition~\ref{sat} because there is only one meaning in the unique state of the left model. Similarly, $(1,1)\Vdash_r\C_a\psi$ because there is only one meaning in each of the two states of the right model. Therefore, statement of the lemma holds in the case when formula $\phi$ has the form $\C_a\psi$.
\end{proof}

The next lemma shows that the left and the right models are distinguishable in the language $\Phi$ of our logical system.

\begin{lemma}
$(1,1)\Vdash_l\K_a p$ and $(1,1)\nVdash_r\K_a p$.
\end{lemma}
\begin{proof}
Note that $1\in \pi^l_1(p)$, see Figure~\ref{K-undefinability figure}.  Thus, $(1,1)\Vdash_l p$ by item 1 of Definition~\ref{sat}. Therefore, $(1,1)\Vdash_l \K_a p$ by item 4 of Definition~\ref{sat}.

At the same time, $1\in \pi^r_1(p)$ and $1\notin \pi^r_2(p)$, see Figure~\ref{K-undefinability figure}. Thus, $(1,1)\Vdash_r p$ and $(2,1)\nVdash_r p$ by item 1 of Definition~\ref{sat}. Therefore, $(1,1)\nVdash_r \K_a p$ by item 4 of Definition~\ref{sat} and because $1\sim_a 2$, see Figure~\ref{K-undefinability figure}.
\end{proof}

The next theorem follows from the two previous lemmas.

\begin{theorem}
Knowledge modality $\K$ is not definable in language $\Phi^{\mhyphen \K}$. \qed
\end{theorem}

\section{Axioms}

In the rest of the paper we give a sound and complete logical system that captures the interplay between knowledge modality $\K$ and comprehension modality $\C$.
In addition to propositional tautologies in language $\Phi$, our logical system contains the following axioms:
\begin{enumerate}
    \item Truth: $\K_a\phi\to\phi$,
    \item Negative Introspection: $\neg\K_a\phi\to\K_a\neg\K_a\phi$,
    \item Distributivity: $\K_a(\phi\to\psi)\to(\K_a\phi\to\K_a\psi)$,
    \item Comprehension of Known: $\K_a\phi\to\C_a\phi$,
    \item Introspection of Comprehension: $\C_a\phi\to\K_a\C_a\phi$,
    \item Comprehension of Negation: $\C_a\phi\to\C_a\neg\phi$,
    \item Comprehension of Implication:\\ $\C_a\phi\to(\C_a\psi\to\C_a(\phi\to\psi))$,
    \item Substitution: $\K_a(\phi\leftrightarrow\psi)\to(\C_a\phi\to\C_a\psi)$,
    \item Comprehension of Comprehension: $\C_a\C_b\phi$,
    \item Incomprehensible: $\C_a(\C_b\phi\to \phi)$.
\end{enumerate}

The Truth, the Negative Introspection, and the Distributivity axioms are standard axioms of epistemic logic S5. The Comprehension of Known axiom states that an agent must comprehend any statement that she knows. The Introspection of Comprehension axiom states that if an agent comprehends a statement, then she must know that she comprehends it. The Comprehension of Negation and the Comprehension of Implication axioms capture the fact that all agents are assumed to understand the meaning of Boolean connectives. Thus, if an agent comprehends $\phi$ and $\psi$, then she must comprehend negation $\neg\phi$ and implication $\phi\to\psi$. The Substitution axiom states that if an agent knows that two sentences are equivalent and she comprehends one of them, then she must comprehend the other. 
The Comprehension of Comprehension axiom states that any agent must comprehend statement $\C_b\phi$, even if she does not comprehend $\phi$. 

The Incomprehensible axiom states that any agent $a$ must comprehend statement $\C_b\phi\to \phi$. We call this axiom  Incomprehensible because we do not have a clear intuition of why it is true. The formal proof of soundness for this axiom is given in Lemma~\ref{Incomprehensible soundness}. 

We write $\vdash\phi$ if formula $\phi$ is provable from the above axioms using the Modus Ponens and the Necessitation inference rules:
$$
\dfrac{\phi, \phi\to\psi}{\psi}
\hspace{15mm}
\dfrac{\phi}{\K_a\phi}.
$$
We write $X\vdash\phi$ if formula $\phi$ is provable from the theorems of our logical system and the set of additional axioms $X$ using only the Modus Ponens inference rule.


\section{Soundness}

The Truth, the Negative Introspection, and the Distributivity axioms are standard axioms of epistemic logic S5. Below we show soundness of each of the remaining axioms as a separate lemma. 

\begin{lemma}
If $(w,m)\Vdash\K_a\phi$, then $(w,m)\vdash\C_a\phi$.
\end{lemma}
\begin{proof}
Consider any state $u\in W$ and any two meanings $m',m''\in M_u$ such that $w\sim_a u$ and $(u,m')\Vdash\phi$. By item 5 of Definition~\ref{sat}, it suffices to show that $(u,m'')\Vdash\phi$.

Note that assumption $(w,m)\Vdash\K_a\phi$ of the lemma implies that $(u,m'')\Vdash\phi$ by item 4 of Definition~\ref{sat} and the assumption $w\sim_a u$.
\end{proof}

\begin{lemma}
If $(w,m)\Vdash\C_a\phi$, then $(w,m)\vdash\K_a\C_a\phi$.
\end{lemma}
\begin{proof}
Consider any state $u$ and any meaning $m'\in M_w$ such that $w\sim_a u$. By item 4 of Definition~\ref{sat}, it suffices to prove that $(u,m')\Vdash\C_a\phi$. Towards this proof, consider any state $v\in W$ and any two meanings $m_1,m_2\in M_v$ such that $u\sim_a v$ and $(v,m_1)\Vdash\phi$. By item 5 of Definition~\ref{sat}, it suffices to show that $(v,m_2)\Vdash\phi$.

Assumptions $w\sim_a u$ and $u\sim_a v$ imply that $w\sim_a v$ because $\sim_a$ is an equivalence relation. Therefore, the assumption $(v,m_1)\Vdash\phi$ implies $(v,m_2)\Vdash\phi$ by item 5 of Definition~\ref{sat} and the assumption $(w,m)\Vdash\C_a\phi$ of the lemma.
\end{proof}

\begin{lemma}
If $(w,m)\Vdash \C_a\phi$, then $(w,m)\Vdash\C_a\neg\phi$.
\end{lemma}
\begin{proof}
Consider any state $u\in W$ and any two meanings $m',m''\in M_u$ such that $w\sim_a u$ and 
\begin{equation}\label{london}
    (u,m')\Vdash\neg\phi.
\end{equation}
Note that by item 5 of Definition~\ref{sat}, it suffices to show that $(u,m'')\Vdash\neg\phi$.

Suppose that $(u,m'')\nVdash\neg\phi$. Thus, $(u,m'')\Vdash\phi$ by item 2 of Definition~\ref{sat}. Hence, $(u,m')\Vdash\phi$ by item 5 of Definition~\ref{sat}, the assumption $(w,m)\Vdash \C_a\phi$ of the lemma, and the assumption $w\sim_a u$. Therefore, $(u,m')\nVdash\neg\phi$ by item 2 of Definition~\ref{sat}, which contradicts statement~(\ref{london}).
\end{proof}

\begin{lemma}
If $(w,m)\Vdash \C_a\phi$ and $(w,m)\Vdash\C_a\psi$, then $(w,m)\vdash\C_a(\phi\to\psi)$.
\end{lemma}
\begin{proof}
Consider any state $u\in W$ and any two meanings $m',m''\in M_u$ such that $w\sim_a u$ and 
\begin{equation}\label{paris}
    (u,m')\Vdash\phi\to\psi.
\end{equation}
Note that by item 5 of Definition~\ref{sat}, it suffices to prove that $(u,m'')\Vdash\phi\to\psi$. Towards this proof, suppose that $(u,m'')\Vdash\phi$. By item 3 of Definition~\ref{sat}, it suffices to show that $(u,m'')\Vdash\psi$.

Assumption $(u,m'')\Vdash\phi$ implies that $(u,m')\Vdash\phi$ by item 5 of Definition~\ref{sat}, the assumption $(w,m)\Vdash \C_a\phi$ of the lemma, and assumption $w\sim_a u$. Hence, $(u,m')\Vdash\psi$ by item 3 of Definition~\ref{sat} and statement~(\ref{paris}). Thus, $(u,m'')\Vdash\psi$, by item 5 of Definition~\ref{sat}, the assumption $(w,m)\Vdash \C_a\psi$ of the lemma, and assumption $w\sim_a u$.
\end{proof}

\begin{lemma}
If $(w,m)\Vdash \K_a(\phi\leftrightarrow\psi)$ and  $(w,m)\Vdash \C_a\phi$, then $(w,m)\Vdash \C_a\psi$.
\end{lemma}
\begin{proof}
Consider any state $u\in W$ and any two meanings $m',m''\in M_u$ such that $w\sim_a u$ and $(u,m')\Vdash\psi$. By item 5 of Definition~\ref{sat}, it suffices to show that $(u,m'')\Vdash\psi$.

By Definition~\ref{sat}, the assumption $(w,m)\Vdash \K_a(\phi\leftrightarrow\psi)$ of the lemma implies that 
\begin{eqnarray}
&&(u,m')\Vdash \psi\to\phi\label{psi to phi},\\
&&(u,m'')\Vdash \phi\to\psi\label{phi to psi}.
\end{eqnarray}
By item 3 of Definition~\ref{sat},  assumption $(u,m')\Vdash\psi$ and statement~(\ref{psi to phi}) imply that $(u,m')\Vdash\phi$. Hence, $(u,m'')\Vdash\phi$ by item 5 of Definition~\ref{sat}, the assumption $(w,m)\Vdash \C_a\phi$ of the lemma, and the assumption $w\sim_a u$. Thus, $(u,m'')\Vdash\psi$ by item 3 of Definition~\ref{sat} and statement~(\ref{phi to psi}).
\end{proof}

\begin{lemma}\label{Comprehension of Comprehension soundness}
$(w,m)\Vdash\C_a\C_b\phi$.
\end{lemma}
\begin{proof}
Consider any state $u\in W$ and any two meanings $m',m''\in M_u$ such that $w\sim_a u$ and $(u,m')\Vdash\C_b\phi$. Note that by item 5 of Definition~\ref{sat}, it suffices to prove that $(u,m'')\Vdash\C_b\phi$. Towards this proof, consider any state $v\in W$ and any two meanings $m_1,m_2\in M_v$ such that $u\sim_b v$ and $(v,m_1)\Vdash\phi$. By item 5 of Definition~\ref{sat}, it suffices to show that $(v,m_2)\Vdash\phi$.
Indeed, by the same item 5 of Definition~\ref{sat}, assumptions  $(u,m')\Vdash\C_b\phi$, $(v,m_1)\Vdash\phi$, and $u\sim_b v$ imply that $(v,m_2)\Vdash\phi$.
\end{proof}


\begin{lemma}\label{Incomprehensible soundness}
$(w,m)\Vdash\C_a(\C_b\phi\to\phi)$.
\end{lemma}
\begin{proof}
Consider any state $u\in W$ and any two meanings $m',m''\in M_u$ such that $w\sim_a u$ and 
\begin{equation}\label{moscow}
    (u,m')\Vdash\C_b\phi\to\phi
\end{equation}
Note that by item 5 of Definition~\ref{sat}, it suffices to prove that $(u,m'')\Vdash\C_b\phi\to\phi$. Towards this proof, suppose that $(u,m'')\Vdash\C_b\phi$. By item 3 of Definition~\ref{sat}, it suffices to show that $(u,m'')\Vdash\phi$. Indeed, by Lemma~\ref{C transfer lemma}, assumption $(u,m'')\Vdash\C_b\phi$ implies that $(u,m')\Vdash\C_b\phi$. It follows by item 3 of Definition~\ref{sat} and statement~(\ref{moscow}) that $(u,m')\Vdash\phi$. Thus, $(u,m'')\Vdash\phi$ by assumption $(u,m'')\Vdash\C_b\phi$, item 5 of Definition~\ref{sat} and because $u\sim_b u$. 
\end{proof}

\section{Completeness Proof Overview}

In this section we sketch a proof of the completeness of our logical system. The complete proof can be found in 
the full version of this paper.

A completeness theorem for a modal logical system is usually proven by constructing a canonical model in which {\em states} are defined to be maximal consistent sets of formulae. This is different in our case, because we define {\em meanings}, rather than states, to be maximal consistent sets of formulae. The set of all such meaning will be denoted by $M$. 

Definition~\ref{model} specifies that any model should have a state-specific set of meanings $M_w$ for each state $w$.  Sets of meanings $M_w$ and $M_u$ corresponding to distinct states $w$ and $u$ can but do not have to be disjoint. In our canonical model they are disjoint. In other words, we partition the set of all meanings (maximal consistent sets of formulae) $M$ into sets of meanings $\{M_w\}_{w\in W}$ corresponding to different states. We define this partition through an equivalence relation $\equiv$ on set $M$. Then, we define {\em states} as equivalence classes of this relation, see Figure~\ref{canonical model figure}.

\begin{figure}[ht]
\begin{center}
\scalebox{0.35}{\includegraphics{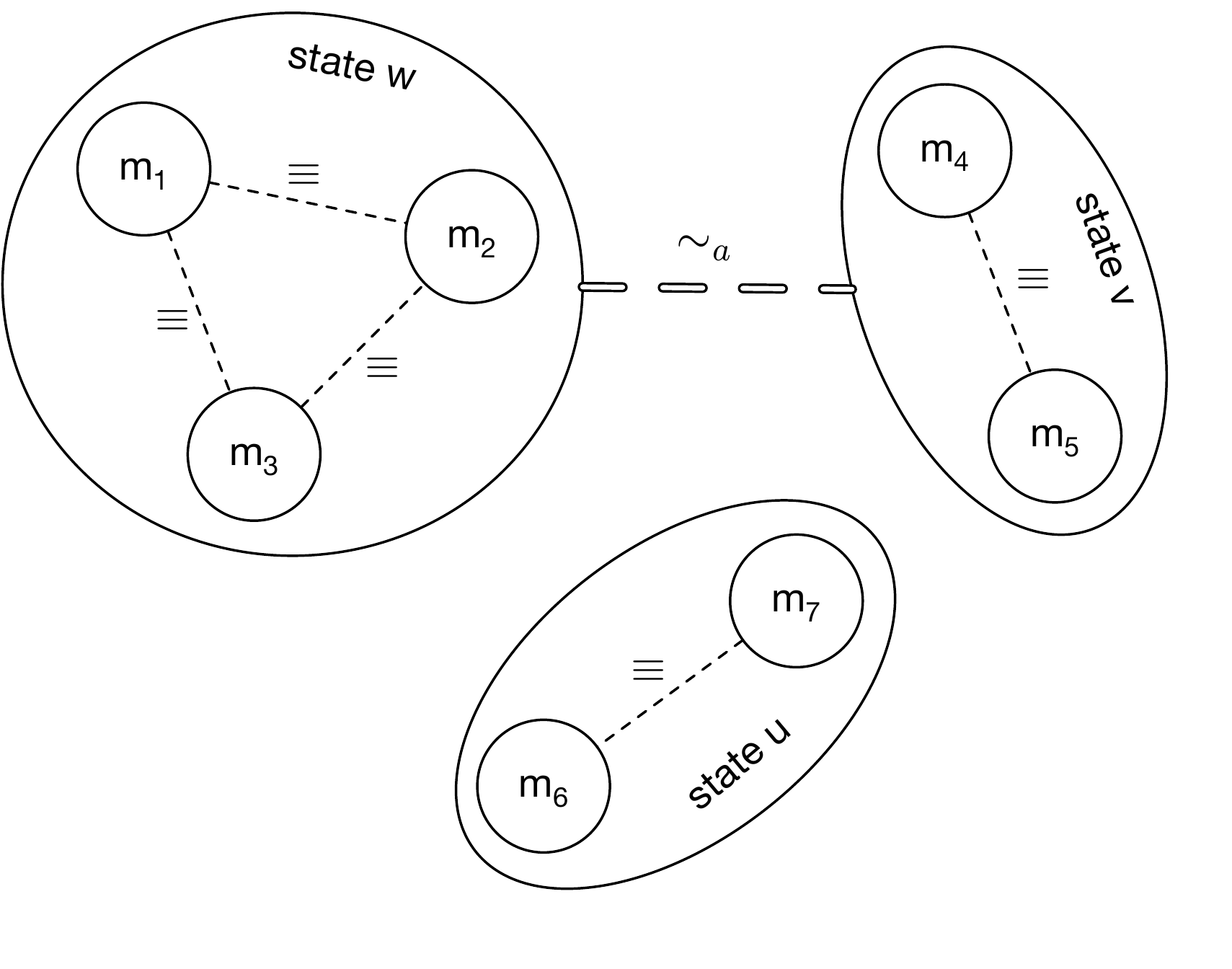}}
\caption{Canonical Model.}\label{canonical model figure}
\end{center}
\end{figure}

The exact definition of relation $\equiv$ is based on the intuition that if $\C_a\phi$ is true under a meaning in a state, then $\phi$ must be consistent across all meanings in the given state. To capture this, we say that $m\equiv m'$ when for each formula $\C_a\phi\in m$, if $\phi\in m$, then $\phi\in m'$, see Definition~7 in the full version.

To define indistinguishability relation $\sim_a$ between states, we first define it as a relation between meanings and then show that this relation is well-defined on states (equivalence classes of meanings with respect to relation $\equiv$). Our definition of indistinguishability of meanings by an agent $a$ is equivalent to the standard approach in epistemic logic: $m\sim_a m'$ if meanings $m$ and $m'$ contain the same $\K$-formulae.

A typical proof of completeness in modal logic includes a step where for each state $w$ that does not contain a modal formula $\Box\phi$ the proof constructs a ``reachable'' state $u$ such that $\neg\phi\in u$. In our proof, such a step for modality $\K$ is very standard and it is described in Lemma~35 of the full version.
The case of modality $\C$, however, is significantly different. Indeed, because item 5 of Definition~\ref{sat} refers to two different meanings, $m'$ and $m''$, the corresponding step for modality $\C$ involves a construction of two maximal consistent sets corresponding to these meanings. Since  $m'$ and $m''$ in item 5 of Definition~\ref{sat} are two meanings in the same state, we must guarantee that $m'\equiv m''$. This means that sets $m'$ and $m''$ must agree on all formulae $\phi$ such that $\C_a\phi$ belongs to at least one of them. 

To construct sets $m'$ and $m''$ for any given formula $\C_a\phi$, we introduce a new technique that we call {\em perfect confirming} sets. First, we define the notion of a confirming set and consider a set $Y$ of formulae that ``must'' belong to both: set $m'$ and $m''$. We show that set $Y$ is confirming. Then, we define {\em perfect} confirming set and show that any confirming set can be extended to a perfect confirming set. We extend set $Y$ to a perfect confirming set $Y'$ and show that sets $Y'\cup\{\phi\}$ and $Y'\cup\{\neg\phi\}$ are consistent. Finally, we use Lindenbaum's lemma to extend sets $Y'\cup\{\phi\}$ and $Y'\cup\{\neg\phi\}$ to maximal consistent sets of formulae $m'$ and $m''$, respectively. The actual proof in the full version of this paper does not define confirming sets directly. Instead, to improve readability, it first defines comprehensible sets and then confirming sets as a class of comprehensible sets.

\section{Conclusion}

The contribution of this paper is three-fold. 
First, we introduced a novel modality ``comprehensible'' and gave its formal semantics in epistemic models with meanings. Second, we have shown that this modality cannot be defined through knowledge modality and vice versa.
Finally, we proposed a sound and complete logical system that describes the interplay between the knowledge and the comprehension modalities. 
In the full version of this paper, we outline a possible extension of our work to a probabilistic setting.

In modal logic, the filtration technique is often used to prove weak completeness of a logical system with respect to a class of finite models~\cite{g72jpl}. Such completeness normally implies decidability of the system. For this approach to work in our case, the class of finite models would require not only the number of states to be finite, but the number of meanings to be finite as well. We have not been successful in adopting the filtration technique to achieve this. Thus, proving decidability of the proposed logical system remains an open question.


%

\bibliography{sp}

\end{document}